\documentclass{article}

\usepackage{geometry}
\usepackage{graphicx}
\usepackage{hyperref}
\hypersetup{
	colorlinks   = true, 
	urlcolor     = blue, 
	linkcolor    = blue, 
	citecolor   = blue 
}

\usepackage{subfigure}
\usepackage{booktabs} 
\usepackage[section]{placeins}
\usepackage{algorithm}
\usepackage{algorithmic}

\usepackage{natbib}
\usepackage{amsmath}
\usepackage{amssymb}
\usepackage{mathtools}
\usepackage{amsthm}

\usepackage{thm-restate}

\usepackage{xcolor}

\DeclareMathOperator*{\argmax}{argmax}

\newcommand{\compilehidecomments}{false}

\newcommand{\bX}{\boldsymbol{X}}

\newcommand{\bx}{\boldsymbol{x}}

\newcommand{\btheta}{\boldsymbol{\theta}}

\newcommand{\relu}{\mathsf{ReLU}}

\newcommand{\sfd}{\mathsf{d}}

\usepackage[capitalize,noabbrev]{cleveref}

\theoremstyle{plain}
\newtheorem{theorem}{Theorem}

\newtheorem{lemma}{Lemma}

\theoremstyle{definition}

\newtheorem{assumption}{Assumption}

\theoremstyle{remark}

\usepackage[textsize=tiny]{todonotes}
\usepackage[affil-it]{authblk}

\title{A Correction of Pseudo Log-Likelihood Method}
\date{}
\author[1]{Shi Feng}
\author[2]{Nuoya Xiong}
\author[3]{Zhijie Zhang}
\author[4]{Wei Chen \thanks{Corresponding author: weic@microsoft.com}}
\affil[1]{Harvard University, MA, USA}
\affil[2]{Institute for Interdisciplinary Information Sciences, Tsinghua University, Beijing, China}
\affil[3]{School of Mathematics and Statistics, Fuzhou University, Fuzhou, China}
\affil[4]{Microsoft Research, Beijing, China}
\begin{document}

\maketitle

\ifthenelse{ \equal{\compilehidecomments}{false} }{%
	\newcommand{\wei}[1]{}
	\newcommand{\shi}[1]{}
	\newcommand{\nuoya}[1]{}
}{
	\newcommand{\wei}[1]{{\color{green}  [\text{Wei:} #1]}}
	\newcommand{\shi}[1]{{\color{yellow!80!black} [\text{Shi:} #1]}}
	\newcommand{\nuoya}[1]{{\color{red} [\text{Nuoya:} #1]}}
}

\begin{abstract}
Pseudo log-likelihood is a type of maximum likelihood estimation (MLE) method used in various fields including contextual bandits, influence maximization of social networks, and causal bandits. However, in previous literature \citep{li2017provably, zhang2022online, xiong2022combinatorial, feng2023combinatorial1, feng2023combinatorial2}, the log-likelihood function may not be bounded, which may result in the algorithm they proposed not well-defined. In this paper, we give a counterexample that the maximum pseudo log-likelihood estimation fails and then provide a solution to correct the algorithms in \citep{li2017provably, zhang2022online, xiong2022combinatorial, feng2023combinatorial1, feng2023combinatorial2}.
\end{abstract}

\section{Problem Description}
\label{sec.problem}

In \citep{li2017provably, zhang2022online, xiong2022combinatorial, feng2023combinatorial1, feng2023combinatorial2}, the authors use the same maximum likelihood estimation (MLE) method. Suppose $X_1, X_2, \ldots, X_d$ are random variables such that $X_i \in [0,1]$. For convenience, we use $\bX$ to denote the vector $(X_1, X_2, \cdots, X_d)$. The function $\mu$ is a monotone increasing and second-order differentiable function from $\mathbb{R}$ to $\mathbb{R}$, and $m(x)$ is defined as: 
\[
m(x)=\begin{cases}
    \int_0^x\mu(x')\sfd x'&x\geq 0\\
    -\int_{-x}^0\mu(x')\sfd x'&x<0
\end{cases}.
\] 
There is a set of parameters $\theta^*_1, \theta^*_2, \ldots, \theta^*_d$, which we also denote by vector $\btheta^*$. Moreover, $Y \in [0,1]$ is the outcome of $X_1, X_2, \ldots, X_d$ such that $\mathbb{E}[Y|\bX] = \mu(\bX^\intercal\btheta^*)$. $Y$ depends on $\bX^\intercal\btheta^*$ and its noise term is independent of $\bX$ and $\btheta^*$.

There are in total $t$ rounds and in the $i^{th}$ round, the values of $X_1, X_2, \ldots, X_d$ are $x_{1,i}, x_{2,i}, \ldots, x_{d,i}$, and the value of $Y$ is $y_i$. For convenience, we denote the vector $(x_{1,i}, x_{2,i}, \ldots, x_{d,i})$ by $\bx_i$. To estimate $\btheta^*$, the MLE method \citep{li2017provably, zhang2022online} takes \begin{align}\hat{\btheta}_t = \argmax_{\btheta \in \mathbb{R}^d} \sum_{i=1}^t \left( y_i \bx_i^\intercal \btheta - m(\bx_i^\intercal \btheta) \right)\label{eq.argmax}\end{align} as the estimation of $\btheta^*$. Since $m$ is a second-order differentiable function, when $\hat{\btheta}_t$ is a maximum of $$\sum_{i=1}^t \left( y_i \bx_i^\intercal \btheta - m(\bx_i^\intercal \btheta) \right),$$ the gradient $\sum_{i=1}^t (y_i - \mu(\bx_i^\intercal \hat{\btheta}_t)) \bx_i$ should be $\mathbf{0}$. In \citep{xiong2022combinatorial, feng2023combinatorial1, feng2023combinatorial2}, they solve $\hat{\btheta}_t$ by the equation $\sum_{i=1}^t (y_i - \mu(\bx_i^\intercal \btheta)) \bx_i = \mathbf{0}$.

Theorem 1 in \citep{li2017provably}, Theorem 3 in \citep{zhang2022onlinearxiv}, and Lemma 1 in \citep{feng2023combinatorial1} provide a guarantee for the distance from $\hat{\btheta}_t$ to $\btheta^*$ when $\hat{\btheta}_t$ exists. However, none of these previous literature \citep{li2017provably, zhang2022online, xiong2022combinatorial, feng2023combinatorial1, feng2023combinatorial2} discusses the existence of $\hat{\btheta}_t$, i.e., whether $\sum_{i=1}^t \left( y_i \bx_i^\intercal \btheta - m(\bx_i^\intercal \btheta) \right)$ could tend to positive infinity when some elements of $\btheta$ tend to infinity. If $\hat{\btheta}_t$ does not exist, the deductions in all the previous literature collapse. In this short paper, we present a counterexample to show that $\hat{\btheta}_t$ may not exist under the conditions of the previous papers \citep{li2017provably, zhang2022online, xiong2022combinatorial, feng2023combinatorial1, feng2023combinatorial2} and provide a solution to address this issue. 
For consistency, we will use the notations in this section throughout this paper. Readers can correlate the notations with those of previous papers.

\section{Counter-Examples}

In this section, we present a counterexample to demonstrate the incorrectness of the default assumption that $\hat{\btheta}_t$ exists. In \citep{li2017provably}, $X_1, \ldots, X_d, Y$ are continuous variables in $[0,1]$. In our counterexample, we let the link function be $\mu(x) = \frac{1}{2 + 2e^{-x}}$, which satisfies Assumptions 1 and 2 in \citep{li2017provably}. For convenience, we denote $\sum_{i=1}^t \left( y_i \bx_i^\intercal \btheta - m(\bx_i^\intercal \btheta) \right)$ by $H(\btheta)$.

When $x_{1,i} = x_{2,i} = \cdots = x_{d,i} > 0$ and $y_i > \frac{1}{2}$ for $i = 1, 2, \ldots, t$, the $j^{th}$ element of the gradient of $H(\btheta)$ satisfies:
\begin{align*}
    \frac{\partial H(\btheta)}{\partial \theta_j} = \left(\sum_{i=1}^t (y_i - \mu(\bx_i^\intercal \hat{\btheta}_t)) \bx_i\right)_j
    &= \sum_{i=1}^t (y_i - \mu(\bx_i^\intercal \hat{\btheta}_t)) x_{j,i} \\
    &> \sum_{i=1}^t \left(\frac{1}{2} - \frac{1}{2 + 2e^{-\bx_i^\intercal \hat{\btheta}_t}}\right) x_{j,i} \\
    &> 0.
\end{align*}
This occurs with a nonzero probability\footnote{In \citep{li2017provably}, the first $\tau$ rounds of Algorithm 1 are random.}, and it indicates that $\lim_{\theta_1 \rightarrow +\infty, \ldots, \theta_d \rightarrow +\infty} H(\btheta) = +\infty$ and thus Eq.\eqref{eq.argmax} does not have a solution. This is a contradiction, and therefore, the proof of Theorem 1 in \citep{li2017provably} collapses.\footnote{Even when $\btheta$ is bounded, this is still a contradiction because the proof of Theorem 1 in \citep{li2017provably} uses $\nabla_{\btheta} H(\btheta) = {\bf 0}$.}

In \citep{zhang2022online, xiong2022combinatorial, feng2023combinatorial1, feng2023combinatorial2}, $X_1, \ldots, X_d, Y$ are binary variables. In our counterexample, we let $\mu(x) = 1 - e^{-x}$, satisfying all assumptions in these papers. When $x_{1,i} = x_{2,i} = \cdots = x_{d,i} = 1$ and $y_i = 1$ for $i = 1, 2, \ldots, t$, the $j^{th}$ element of the gradient of $H(\btheta)$ satisfies:
\begin{align*}
    \frac{\partial H(\btheta)}{\partial \theta_j} = \left(\sum_{i=1}^t (y_i - \mu(\bx_i^\intercal \hat{\btheta}_t)) \bx_i\right)_j
    &= \sum_{i=1}^t (y_i - \mu(\bx_i^\intercal \hat{\btheta}_t)) x_{j,i} \\
    &> \sum_{i=1}^t e^{-\bx_i^\intercal \hat{\btheta}_t} > 0.
\end{align*}
This also occurs with a nonzero probability and it indicates that $\lim_{\theta_1 \rightarrow +\infty, \ldots, \theta_d \rightarrow +\infty} H(\btheta) = +\infty$ and thus Eq.\eqref{eq.argmax} does not have a solution. This is a contradiction, and therefore, the proofs of Theorem 3 in \citep{zhang2022online} and Lemma 1 in \citep{feng2023combinatorial1} collapses. Furthermore, the regret analysis in \citep{zhang2022online, xiong2022combinatorial, feng2023combinatorial1, feng2023combinatorial2} collapse.\footnote{In \citep{xiong2022combinatorial}, only the regret analysis of Algorithm 1 uses the MLE method in Eq.\eqref{eq.argmax} and collapses due to this counterexample.}

\section{The Solution}

In this section, we provide a feasible solution to promise the existence of $\hat{\btheta}_t$, which fixes the issues in the analysis of the five previous papers we mentioned in Section~\ref{sec.problem}. Before our solution, one should notice that the constraint on $\mu$ in all five previous papers is Assumptions $1$ and $2$ in \citep{li2017provably}.
\begin{assumption}[Assumption 1 in \citep{li2017provably}]
\label{asm.1}
    $\kappa:=\mathsf{inf}_{\bx\in[0,1]^d,\left\|\btheta-\btheta^*\right\|\leq 1}\mu'(\bx^\intercal\btheta)>0$.\footnote{\cite{li2017provably} use a weaker version with $\left\|\bx\right\|\leq 1$. 
    	The other four papers use the current stronger version.   
}
\end{assumption}

\begin{assumption}[Assumption 2 in \citep{li2017provably}]
\label{asm.2}
    Function $\mu $ is twice differentiable. Its first and second-order derivatives are upper-bounded by $L_\mu$ and $M_\mu$, respectively.
\end{assumption}

In summary, our solution is replacing function $\mu$ by another function $h$ in the algorithm such that the following conditions hold: $\lim_{x\rightarrow +\infty}h(x)=+\infty$, $\lim_{x\rightarrow -\infty}h(x)=-\infty$, $h$ is monotone increasing and twice differentiable, $h$ satisfies Assumptions~\ref{asm.1} and \ref{asm.2}, and when $x$ is in the range $\left[-\sum_{i=1}^d\relu(-\theta^*_i),\sum_{i=1}^d\relu(\theta^*_i)\right]$\footnote{$\relu(x)=\max\{0,x\}$ and $\bX^\intercal\btheta^*$ has to be in this interval.}, $h(x)=\mu(x)$. 

We firstly prove that $\mu$ can be converted to a monotone increasing function $g$ satisfying $\lim_{x\rightarrow+\infty}g(x)=+\infty$. If we already have $\lim_{x\rightarrow+\infty}\mu(x)=+\infty$, we can directly let $g\equiv \mu$. Otherwise, $\mu$ has an upper bound. In \citep{zhang2022online, xiong2022combinatorial, feng2023combinatorial1, feng2023combinatorial2}, $\theta^*_i\in[0,1]$ so each $\theta^*_i$ is bounded. In \citep{li2017provably}, since $\mu$ is upper-bounded, we know that $\lim_{x\rightarrow+\infty}\mu'(x)=0$. 
According to Assumption~\ref{asm.1}, $\theta^*_i\leq \max\{x:\mu'(x)=\kappa\}-1$.\footnote{Otherwise, when $x_j=0$ if $j\neq i$, $x_i=1$ and $\theta_i^*>\max\{x:\mu'(x)=\kappa\}-1, \theta_i=\theta_i^*+1>\max\{x:\mu'(x)=\kappa\}$, we have $\bx^\intercal\btheta>\max\{x:\mu'(x)=\kappa\}$ and thus $\mu'(\bx^\intercal\btheta)<\kappa$, which is a contradiction to Assumption~\ref{asm.1}.} \wei{I did not get this.}\shi{I added a footnote.}
Therefore, $\sum_{i=1}^d\relu(\theta^*_i)$ is also upper-bounded in this case, we denote the upper bound by $U$. We find a $x^*\geq U+d$ such that $\mu''(x^*)<0$. If $x^*$ does not exist, we know that $\mu(x)\geq \mu(U+d)+\mu'(U+d)(x-(U+d))$ when $x\geq U+d$, which is contradictory to $\mu$ is upper-bounded. Hence, when $\mu$ is upper-bounded,
we define the conversion as\begin{align*}
    g(x)=\begin{cases}
        \mu(x)&x\leq x^*\\
        \mu(x^*)+\frac{\mu'(x^*)^2}{\mu''(x^*)}\ln\left(-\frac{\mu'(x^*)}{\mu''(x^*)}\right)-\frac{\mu'(x^*)^2}{\mu''(x^*)}\ln\left(x-x^*-\frac{\mu'(x^*)}{\mu''(x^*)}\right)&x>x^*
    \end{cases}.
\end{align*}

\begin{lemma}
\label{lemma.1}
    By doing the conversion above, we can replace function $\mu$ by $g$ such that $\lim_{x\rightarrow +\infty}g(x)=+\infty$, $g$ is monotone increasing and twice differentiable, $g$ satisfies Assumptions~\ref{asm.1} and \ref{asm.2}, and when $x$ is in the range $\left[-\sum_{i=1}^d\relu(-\theta^*_i),\sum_{i=1}^d\relu(\theta^*_i)\right]$, $g(x)=\mu(x)$. 
\end{lemma}

\begin{proof}
    When $\mu$ is not upper-bounded, we let $g\equiv \mu$ so the claim is proved.
    When $\mu$ is upper-bounded, during the producing process of $Y$, the input of $\mu$ is $\bX\cdot\btheta^*$, which is in the range $\left[-\sum_{i=1}^d\relu(-\theta^*_i),\sum_{i=1}^d\relu(\theta^*_i)\right]\subseteq (-\infty, U]$. Hence, when we replace $\mu$ with $g$, the joint conditional distribution of $Y$ on $\bX$ is not impacted.

    Moreover, we can compute that \begin{align*}
    g'(x)=\begin{cases}
        \mu'(x)&x\leq x^*\\
        -\frac{\mu'(x^*)^2}{\mu''(x^*)\left(x-x^*-\frac{\mu'(x^*)}{\mu''(x^*)}\right)}&x>x^*
    \end{cases},
\end{align*}
and\begin{align*}
    g''(x)=\begin{cases}
        \mu''(x)&x\leq x^*\\
        \frac{\mu'(x^*)^2}{\mu''(x^*)\left(x-x^*-\frac{\mu'(x^*)}{\mu''(x^*)}\right)^2}&x>x^*
        \end{cases}.
\end{align*}
Therefore, we have $\lim_{x\rightarrow{x^*}^+}g(x)=\mu(x^*)$ and $\lim_{x\rightarrow{x^*}^-}g(x)=\mu(x^*)$. Hence, $g$ is continuous. Moreover, $\lim_{x\rightarrow{x^*}^+}g'(x)=\mu'(x^*)=\lim_{x\rightarrow{x^*}^-}g'(x)$ and $\lim_{x\rightarrow{x^*}^+}g''(x)=\mu''(x^*)=\lim_{x\rightarrow{x^*}^-}g''(x)$, so $g(x)$ is twice differentiable and $g''$ is continuous.

Now we only need to verify Assumptions~\ref{asm.1} and \ref{asm.2}. Firstly, when $x>x^*$, we have $g'(x)<g'(x^*)=\mu'(x^*)\leq L\mu$ and $g''(x)<g''(x^*)=\mu''(x^*)\leq M_{\mu}$, so Assumption~\ref{asm.2} holds. Secondly, $\max_{\bx\in[0,1]^{d},\left\|\btheta-\btheta^*\right\|\leq 1}\bx\cdot \btheta\leq U+d\leq x^*$, so the conversion does not impact the value of $\kappa$. The claim then follows.
\end{proof}

We secondly prove that $g$ can be converted to a monotone increasing function $h$ satisfying $\lim_{x\rightarrow+\infty}h(x)=+\infty$ and simultaneously, $\lim_{x\rightarrow-\infty}h(x)=-\infty$. If we already have $\lim_{x\rightarrow-\infty}\mu(x)=-\infty$, we can directly let $h\equiv g$. Otherwise, $\mu$ has a lower bound. In \citep{zhang2022online, xiong2022combinatorial, feng2023combinatorial1, feng2023combinatorial2}, $\theta^*_i\in[0,1]$ so each $\theta^*_i$ is bounded. In \citep{li2017provably}, since $\mu$ is lower-bounded, we know that $\lim_{x\rightarrow-\infty}\mu'(x)=0$. According to Assumption~\ref{asm.1}, $\theta^*_i\geq \min\{x:\mu'(x)=\kappa\}+1$. \footnote{Otherwise, when $x_j=0$ if $j\neq i$, $x_i=1$ and $\theta_i^*<\min\{x:\mu'(x)=\kappa\}+1, \theta_i=\theta_i^*-1<\min\{x:\mu'(x)=\kappa\}$, we have $\bx^\intercal\btheta<\min\{x:\mu'(x)=\kappa\}$ and thus $\mu'(\bx^\intercal\btheta)<\kappa$, which is a contradiction to Assumption~\ref{asm.1}.} Therefore, $-\sum_{i=1}^d\relu(-\theta^*_i)$ is also lower-bounded in this case, we denote the lower bound by $L$. We find a $x^{**}\leq L-d$ such that $\mu''(x^{**})>0$. If $x^{**}$ does not exist, we know that $\mu(x)\leq \mu(L-d)+\mu'(L-d)(x-(L-d))$ when $x\leq L-d$, which is contradictory to $\mu$ is lower-bounded. Hence, when $\mu$ is lower-bounded, We define the conversion as\begin{align*}
    h(x)=\begin{cases}
        g(x)&x\geq x^{**}\\
        g(x^{**})-\frac{g'(x^{**})^2}{g''(x^{**})}\ln\left(-\frac{g'(x^{**})}{g''(x^{**})}\right)+\frac{g'(x^{**})^2}{g''(x^{**})}\ln\left(-x+x^{**}+\frac{g'(x^{**})}{g''(x^{**})}\right)&x<x^{**}
    \end{cases}.
\end{align*}

\begin{lemma}
\label{lemma.2}
    By doing the conversion above, we can replace function $\mu$ by $h$ such that $\lim_{x\rightarrow +\infty}h(x)=+\infty$, $\lim_{x\rightarrow -\infty}h(x)=-\infty$, $h$ is monotone increasing and twice differentiable, $h$ satisfies Assumptions~\ref{asm.1} and \ref{asm.2}, and when $x$ is in the range $\left[-\sum_{i=1}^d\relu(-\theta^*_i),\sum_{i=1}^d\relu(\theta^*_i)\right]$, $h(x)=\mu(x)$.
\end{lemma}

\begin{proof}
When $\mu$ is not lower-bounded, we let $h\equiv g$ so the claim is proved by Lemma~\ref{lemma.1}.
    When $\mu$ is lower-bounded, during the producing process of $Y$, the input of $\mu$ is $\bX\cdot\btheta^*$, which is in the range $\left[-\sum_{i=1}^d\relu(-\theta^*_i),\sum_{i=1}^d\relu(\theta^*_i)\right]\subseteq [-L, +\infty)$. Hence, combining Lemma~\ref{lemma.1}, when we replace $\mu$ with $h$, the joint conditional distribution of $Y$ on $\bX$ is not impacted.

Moreover, we can compute that \begin{align*}
    h'(x)=\begin{cases}
        g'(x)&x\geq x^{**}\\
        \frac{g'(x^{**})^2}{g''(x^{**})\left(-x+x^{**}+\frac{g'(x^{**})}{g''(x^{**})}\right)}&x<x^{**}
    \end{cases},
\end{align*}
and\begin{align*}
    h''(x)=\begin{cases}
        g''(x)&x\geq x^{**}\\
        \frac{g'(x^{**})^2}{g''(x^{**})\left(x-x^{**}-\frac{g'(x^{**})}{g''(x^{**})}\right)^2}&x<x^{**}
        \end{cases}.
\end{align*}
Therefore, we have $\lim_{x\rightarrow{x^{**}}^+}h(x)=g(x^{**})$ and $\lim_{x\rightarrow{x^{**}}^-}h(x)=g(x^{**})$. Hence, $h$ is continuous. Moreover, $\lim_{x\rightarrow{x^{**}}^+}h'(x)=g'(x^{**})=\lim_{x\rightarrow{x^{**}}^-}h'(x)$ and $\lim_{x\rightarrow{x^{**}}^+}h''(x)=g''(x^{**})=\lim_{x\rightarrow{x^{**}}^-}h''(x)$, so $h(x)$ is twice differentiable and $h''$ is continuous.

Now we only need to verify Assumptions~\ref{asm.1} and \ref{asm.2}. Firstly, when $x<x^{**}$, we have $h'(x)<h'(x^{**})=g'(x^{**})\leq L_{\mu}$ and $h''(x)<h''(x^{**})=g''(x^{**})\leq M_{\mu}$, so Assumption~\ref{asm.2} holds. Secondly, $\min_{\bx\in[0,1]^{d},\left\|\btheta-\btheta^{*}\right\|\leq 1}\bx\cdot \btheta\geq L-d\geq x^{**}$, so the conversion does not impact the value of $\kappa$. Until now, the claim has been proven.
\end{proof}

Hence, we can replace $\mu$ by $h$ in the MLE method in \eqref{eq.argmax} without impacting the data distribution and our requirements on $\mu$. Finally, we prove that by using $h$ in \eqref{eq.argmax}, the existence of $\hat{\btheta}_t$ is promised. let 
\[m_h(x) = \begin{cases}
    \int_0^xh(x')\sfd x' &x\geq 0,\\
    -\int_{-x}^0h(x')\sfd x'&x<0.
\end{cases}\]

\wei{The following theorem sounds like a lemma because it puts those assumptions on $h$ as a condition. A complete final theorem would combine the previous lemmas of constructing $h$ and simply say that 
	with that construction, the new maximum exists.}
\shi{I added a main theorem.}
\begin{lemma}
\label{lemma.solution}
    When $h$ is monotone increasing, $\lim_{x\rightarrow+\infty}h(x)=+\infty$, and $\lim_{x\rightarrow-\infty}h(x)=-\infty$, the maximum of $
        H_h(\btheta)=\sum_{i=1}^t \left( y_i \bx_i^\intercal \btheta - m_h(\bx_i^\intercal \btheta) \right)$
    exists.
\end{lemma}

\begin{proof}
    We only need to prove that $$H_h(\btheta)=\sum_{i=1}^{t}\left(y_i\bx_i^\intercal\btheta-m_h(\bx_i^\intercal\btheta)\right)$$ is a concave function with respect to $\btheta$ and $\lim_{(\btheta)_j\rightarrow\infty}H_h(\btheta)=-\infty$ or $\frac{\partial H_h(\btheta)}{\partial (\btheta)_j}\equiv0$ for all $j\in[d]$, which implies that $H$ has a maximal point. Firstly, we know that \begin{align*}
        \frac{\partial^2 m_h(x)}{\partial x^2}=h'(x)>0,
    \end{align*}
    so $m_h$ is a convex function. Therefore, for any vectors $\btheta_1, \btheta_2\in \mathbb{R}^{d}$ and $\lambda\in[0,1]$, we have \begin{align*}
        m_h\left( \bx_i^\intercal(\lambda\btheta_1+(1-\lambda)\btheta_2)\right)&=m_h\left( \lambda\bx_i^\intercal\btheta_1+(1-\lambda)\bx_i^\intercal\btheta_2)\right)\\
        &\leq \lambda m_h(\bx_i^\intercal\btheta_1)+(1-\lambda)m_h(\bx_i^\intercal\btheta_2),
    \end{align*}
    so $m_h(\bx_i^\intercal\btheta)$ is also a convex function with respect to $\btheta$ and the Hessian matrix $\mathbf{H}[m_h(\bx_i^\intercal\btheta)]$ of $m_h(\bx_i^\intercal\btheta)$ with respect to $\btheta$ should be positive semidefinite. Now we can compute the Hessian matrix $\mathbf{H}[H_h(\btheta)]$ as\begin{align*}
        \mathbf{H}[H_h(\btheta)]=\sum_{i=1}^t\left(-\bx_i^\intercal\bx_i\cdot\mathbf{H}[m_h(\bx_i^\intercal\btheta)]\right).
    \end{align*}
    Hence, $\mathbf{H}[H_h(\btheta)]$ is negative semidefinite because multiplying a positive semidefinite matrix by a negative scalar preserves the semidefiniteness. Thus $H$ is a concave function with respect to $\btheta$.

    Now for any $j\in[d]$, we prove that $\lim_{(\btheta_{X})_j\rightarrow+\infty}H_h(\btheta_ X)=-\infty$ and $\lim_{(\btheta_{X})_j\rightarrow-\infty}H_h(\btheta_ X)=-\infty$ or $\frac{\partial H_h(\btheta)}{\partial (\btheta)_j}\equiv 0$. Firstly, we have\begin{align*}
        \frac{\partial H_h(\btheta)}{\partial (\btheta)_j}&=\sum_{i=1}^t\left(y_i(\bx_i)_j-(\bx_i)_jm_h'(\bx_i^\intercal \btheta)\right)\\
        &=\sum_{i=1}^t\left(y_i(\bx_i)_j-(\bx_i)_jh(\bx_i^\intercal \btheta)\right).
    \end{align*}
    If $(\bx_i)_j=0$ for all $i\in[t]$, we have $\frac{\partial H_h(\btheta)}{\partial (\btheta)_j}\equiv 0$. Otherwise, we have\begin{align*}
        \lim_{(\btheta_{X})_j\rightarrow+\infty}\frac{\partial H_h(\btheta)}{\partial (\btheta)_j}&=\lim_{(\btheta_{X})_j\rightarrow+\infty}\sum_{i=1}^t\left(y_i(\bx_i)_j-(\bx_i)_jh(\bx_i^\intercal \btheta)\right)\\
        &=\lim_{(\btheta_{X})_j\rightarrow+\infty}\sum_{i=1}^t(\bx_i)_j\left(y_i-h(\bx_i^\intercal \btheta)\right)\\
        &=-\infty,\tag{$\lim_{(\btheta_{X})_j\rightarrow+\infty}h(\bx_i^\intercal \btheta)=+\infty$}
    \end{align*}
    which indicates that $\lim_{(\btheta_{X})_j\rightarrow+\infty}H_h(\btheta_ X)=-\infty$. Also, we have\begin{align*}
        \lim_{(\btheta_{X})_j\rightarrow-\infty}\frac{\partial H_h(\btheta)}{\partial (\btheta)_j}&=\lim_{(\btheta_{X})_j\rightarrow-\infty}\sum_{i=1}^t\left(y_i(\bx_i)_j-(\bx_i)_jh(\bx_i^\intercal \btheta)\right)\\
        &=\lim_{(\btheta_{X})_j\rightarrow-\infty}\sum_{i=1}^t(\bx_i)_j\left(y_i-h(\bx_i^\intercal \btheta)\right)\\
        &=+\infty,\tag{$\lim_{(\btheta_{X})_j\rightarrow-\infty}h(\bx_i^\intercal \btheta)=-\infty$}
    \end{align*}
    which indicates that $\lim_{(\btheta_{X})_j\rightarrow-\infty}H_h(\btheta_ X)=-\infty$. 

    Therefore, we have proved that $H_h(\btheta)$ has at least one global maximum, which indicates that the equation has at least one solution.
\end{proof}

In conclusion, by combining Lemmas~\ref{lemma.1}, \ref{lemma.2}, and \ref{lemma.solution}, we arrive at the following main theorem:

\begin{theorem}[Main Theorem]
Consider a monotone-increasing and twice differentiable function $\mu$ that satisfies Assumptions~\ref{asm.1} and~\ref{asm.2}. By transforming $\mu$ into the function $h$ as described above, we ensure that $h$ remains monotone-increasing and twice differentiable. Furthermore, $h$ preserves the same mapping in the range of input $\bx^\intercal\btheta^*$ and conforms to Assumptions~\ref{asm.1} and~\ref{asm.2}. Most importantly, the maximum of
$H_h(\btheta)=\sum_{i=1}^t \left( y_i \bx_i^\intercal \btheta - m_h(\bx_i^\intercal \btheta) \right)$
exists.
\end{theorem}

\begin{proof}
    According to Lemmas~\ref{lemma.1} and~\ref{lemma.2}, it is established that $h$ remains monotone-increasing and twice differentiable. Additionally, these lemmas confirm that $h$ preserves the same mapping in the range of input $\bx^\intercal\btheta^*$ and conforms to Assumptions~\ref{asm.1} and~\ref{asm.2}. Furthermore, Lemmas~\ref{lemma.1} and~\ref{lemma.2} demonstrate that $\lim_{x\rightarrow+\infty}h(x)=+\infty$ and $\lim_{x\rightarrow -\infty}h(x)=-\infty$, satisfying the condition of Lemma~\ref{lemma.solution}. Therefore, we can conclude that the maximum of
$H_h(\btheta)=\sum_{i=1}^t \left( y_i \bx_i^\intercal \btheta - m_h(\bx_i^\intercal \btheta) \right)$
indeed exists.
\end{proof}
    
\section{Conclusion}
In this paper, we identify and address an issue present in a type of Maximum Likelihood Estimation (MLE) method commonly employed in previous studies. Individuals intending to use similar methods in the future should be cautious of the same or analogous issues. Furthermore, one might consider a more intuitive and efficient solution to ensure the existence of $\hat{\btheta}_t$, rather than artificially constructing a new function.

\clearpage
\bibliography{main}
\bibliographystyle{abbrvnat}
\end{document}